\documentclass{article}

 \usepackage[dblblindworkshop, final]{neurips_2025}
\workshoptitle{GenAI in Finance}

\usepackage[utf8]{inputenc} 
\usepackage[T1]{fontenc}    
\usepackage{hyperref}       
\usepackage{url}            
\usepackage{booktabs}       
\usepackage{amsfonts}       
\usepackage{nicefrac}       
\usepackage{microtype}      
\usepackage{xcolor}         

\usepackage{amsmath}
\usepackage{amssymb}
\usepackage{amsthm}

\usepackage{enumerate}
\usepackage{graphicx}
\usepackage{multicol, multirow}
\usepackage{threeparttable}
\usepackage[figuresright]{rotating}
\usepackage[ruled,linesnumbered]{algorithm2e}
\usepackage[justification=centering]{caption}
\newtheorem{theorem}{Theorem}[section]

\newtheorem{proposition}[theorem]{Proposition}

\theoremstyle{definition}

\newtheorem{definition}{Definition}[section]

\theoremstyle{remark}
\newtheorem{remark}{Remark}[section]

\usepackage{amsmath,amsfonts,bm,bbm}

\def\rd{{\mathrm{d}}}

\DeclareMathAlphabet{\mathsfit}{\encodingdefault}{\sfdefault}{m}{sl}
\SetMathAlphabet{\mathsfit}{bold}{\encodingdefault}{\sfdefault}{bx}{n}

\def\gA{{\mathcal{A}}}

\def\gF{{\mathcal{F}}}

\def\gL{{\mathcal{L}}}

\def\gX{{\mathcal{X}}}

\def\0{{\bf 0}}
\def\1{{\bf 1}}

\def\PB{{\mathbb P}}

\def\argmax{\mathop{\rm argmax}}

\def\ent{\operatorname{Ent}}

\newcommand{\R}{\mathbb{R}}
\newcommand{\E}{\mathbb{E}}
\newcommand{\Var}{\mathrm{Var}}

\newcommand{\MV}{\operatorname{MV}}
\newcommand{\OCE}{\operatorname{OCE}}
\newcommand{\Ent}{\operatorname{Ent}}
\newcommand{\HD}{\text{HD}}
\newcommand{\mkv}{\text{Mkv}}
\newcommand{\aug}{\text{aug}}

\title{Risk-Sensitive Q-Learning in Continuous Time with Application to Dynamic Portfolio Selection}

\author{%
  Chuhan~Xie \\
  School of Mathematical Sciences \\
  Peking University\\
  Beijing, 100871 \\
  \texttt{ch\_xie@pku.edu.cn} \\
}

\begin{document}

\maketitle

\begin{abstract}
  This paper studies the problem of risk-sensitive reinforcement learning (RSRL) in continuous time, where the environment is characterized by a controllable stochastic differential equation (SDE) and the objective is a potentially nonlinear functional of cumulative rewards. We prove that when the functional is an optimized certainty equivalent (OCE), the optimal policy is Markovian with respect to an augmented environment. We also propose \textit{CT-RS-q}, a risk-sensitive q-learning algorithm based on a novel martingale characterization approach. Finally, we run a simulation study on a dynamic portfolio selection problem and illustrate the effectiveness of our algorithm.
\end{abstract}

\section{Introduction}

Traditional reinforcement learning (RL) algorithms typically aim to maximize the expected cumulative reward in a discrete-time setting. However, they often struggle in environments that are inherently continuous, or in applications where the full distributional characteristics of returns are of particular concern.
A representative example arises in high-frequency trading: since trades occur irregularly over time, regular discretization of timestamps may cause potential information loss and training instability. Moreover, while achieving higher profits is desirable, it should not come at the cost of excessive volatility or drawdowns. Hence, the variance of the trader’s profit, alongside its expectation, becomes a crucial performance criterion.

Existing research has addressed each of these two challenges separately. For the first, \citet{wang2020reinforcement,jia2022policy1,jia2022policy2,jia2023q,zhao2023policy,tang2022exploratory} rigorously extend several RL algorithms to a continuous-time framework. For the second, numerous studies on risk-sensitive learning attempt to optimize risk measures beyond the expectation operator \citep{mihatsch2002risk,geibel2005risk,shen2014risk,fei2020risk,lutjens2019safe,garcia2015comprehensive}. Likewise, distributional reinforcement learning (DRL) focuses on learning the entire return distribution rather than specific moments or functionals \citep{bellemare2017distributional,dabney2018distributional,rowland2019statistics,bellemare2023distributional}. These approaches have achieved notable success both theoretically and empirically.

However, few studies have explored scenarios where both challenges coexist. This motivates our work, which takes an initial step toward unifying the two perspectives. We aim to bridge their intrinsic connections and establish a rigorous methodological foundation for algorithms that are simultaneously compatible with continuous-time modeling and risk-sensitive objectives. This paper provides a conceptual overview of the core ideas underlying continuous-time risk-sensitive RL, and demonstrates the feasibility of algorithms inspired by these principles.

\section{Problem Formulation}
\label{sec: problem formulation}

We first state our formulation of continuous-time risk-sensitive RL. Our problem is to control the state dynamics governed by a stochastic differential equation (SDE), defined on a filtered probability space $(\Omega, \gF, (\gF_s)_{s\geq 0}, \PB)$ along with a standard $n$-dimensional Brownian motion $W=\{W_s, s\geq 0\}$:
\begin{align}
\label{eq: sde}
    \rd X_s^{\pi} = \mu(s, X_s^{\pi}, a^{\pi}_s) \rd s + \sigma(s, X_s^{\pi}, a^{\pi}_s) \rd W_s,\quad s\in [t,T].
\end{align}
Here, $\mu\colon [0,T]\times \R^d \times \gA \rightarrow \R^d$ and $\sigma \colon [0,T] \times \R^d \times \gA \rightarrow \R^{d\times n}$ are given functions, where $\gA\subset \R^m$ is the action space. 
The agent's action $a^{\pi}_s$ is sampled independently of the Brownian motion according to a (possibly history-dependent) policy $\pi(\cdot \mid \gF_s^X)$, where $\gF_s^X$ refers to the natural filtration containing $(X_u^\pi)_{t\leq u\leq s}$ and $(a_u^\pi)_{t\leq u < s}$.

Let $r\colon [0,T] \times \R^d \times \gA \rightarrow \R$ be an instantaneous reward function, $h\colon \R^d \rightarrow \R$ be the lump-sum reward function applied at the end of the period, and $\delta > 0$ be a discount factor that measures the time-value of the payoff. By simulating the process \eqref{eq: sde} from $X_t^{\pi} = x$, we eventually receive a sum of discounted rewards:
\begin{align}
    Z^{\pi}(t,x) &= \int_t^T e^{-\delta(s-t)}r(s, X_s^{\pi}, a^{\pi}_s) \rd s + e^{-\delta (T-t)} h(X_T^{\pi}).
\end{align}
Given a risk measure $U\colon L^2(\Omega, \gF, \PB) \rightarrow \R$, our goal is to find the optimal policy $\pi^*$ that maximizes the risk-sensitive objective with entropy regularization: 
\begin{align}
\label{eq: J0}
    J_0^{\pi}(t, x) = U(Z^{\pi}(t,x)) + \tau\ent(\pi;t,x),
\end{align}
where $\Ent(\pi;t,x) = -\E\left[\int_t^T e^{-\delta (s-t)} \log \pi(a_s^{\pi}\mid \gF_s^X) \rd s \mid X_t^{\pi} = x\right]$ is defined as the discounted cumulative entropy of the policy $\pi$ along the process starting from $X_t^\pi = x$.


\section{Resolving Nonlinearity of the Functional Objective}
\label{sec: resolve nonlinear}

When $U$ is not the expectation operator as in traditional RL \citep{jia2022policy1,jia2022policy2,jia2023q}, typical algorithms based on Bellman equations fail due to its intrinsic nonlinearity. In fact, the optimal policy is generally not Markovian \citep{wang2024reductions}, which renders it complex to find the optimal policy.

To tackle this issue, we consider a special kind of risk measures called optimized certainty equivalents (OCEs), and show that the optimal policy is Markovian with respect to an augmented SDE in this case. In this way, the original problem breaks down to a conventional policy optimization task followed by a reverse transformation of the learned policy from the augmented SDE to the original one.

\subsection{OCE risk measures}
\label{sec: oce rm}

Optimized certainty equivalents (OCEs) are special functionals on random variables that admit a variational expression with an expectation operator:
\begin{align}
    U(W) = \OCE_{\varphi}(W) = \sup_{\eta \in \R} \left\{ \eta + \E[\varphi(W - \eta)] \right\},
\end{align}
where $W\in L^2(\Omega, \gF, \PB)$ is a random variable, and $\varphi \colon \R \rightarrow \R$ is a concave utility function. 

Table~\ref{tab: OCE} lists a series of common OCE risk measures along with their utility functions, including linear/exponential/logarithm utilities, CVaR risk, etc. The concept generalizes the classical notion of certainty equivalents by incorporating a utility function and an optimization procedure to unify a wider range of risk measures \citep{ben2007old}, and has been widely adopted in the optimization and mathematical finance literature.

\subsection{Markovian optimality}

It can be shown that when the risk measure $U$ is an OCE with respect to the utility function $\varphi$, the optimal policy is Markovian with respect to the following augmented SDE:
\begin{align}
\label{eq: augmented sde}
    \rd \underbrace{
    \begin{pmatrix}
        X_s^\pi \\ B_{0,s}^\pi \\ B_{1,s}^\pi
    \end{pmatrix}
    }_{X_{\aug, s}^\pi} = \underbrace{
    \begin{pmatrix}
        {\mu}(s, {X}_s^\pi, a_s^\pi) \\ {B}_{1,s}^\pi {r}(s, {X}_s^\pi, a_s^\pi) \\ -\delta {B}_{1,s}^\pi
    \end{pmatrix}
    }_{{\mu}_{\aug}(s, {X}_s^\pi, {B}_{0,s}^\pi, {B}_{1,s}^\pi, a_s^\pi)} \rd s + \underbrace{
    \begin{pmatrix}
        {\sigma}(s, {X}_s^\pi, a_s^\pi) \\ 0_n^\top \\ 0_n^\top
    \end{pmatrix}
    }_{{\sigma}_{\aug}(s, {X}_s^\pi, {B}_{0,s}^\pi, {B}_{1,s}^\pi, a_s^\pi)} \rd W_s, \quad s\in [t,T],
\end{align}
with the instantaneous reward function $r_{\aug}(t, x, b_0, b_1, a) \equiv 0$ and the lump-sum reward function $h_{\aug}(x, b_0, b_1) = \varphi(b_0 + b_1 h(x) )$. In the sense of conventional RL, the entropy-regularized value function related to \eqref{eq: augmented sde} is as follows (see Appendix \ref{app: derive aug value}):
\begin{align}
\label{eq: J}
    J^{\pi}(t,x,b_0,b_1) = \E[\varphi(b_0 + b_1 Z^{\pi}(t,x))] + \tau b_1 \Ent(\pi; t,x).
\end{align}

Compared with \eqref{eq: sde}, the augmented SDE \eqref{eq: augmented sde} has two additional states $(B_{0,s}^\pi, B_{1,s}^\pi)$ that respectively track the cumulative reward so far and the effect of the discount factor \citep{bauerle2011markov,bauerle2021minimizing,wang2024reductions}.

\begin{proposition}
\label{prop: markovian policy}
    If the risk measure $U$ is an OCE with respect to the utility function $\varphi$, the optimal policy $\pi_0^* = \argmax_{\pi} J^{\pi}_0(t,x)$ of the SDE \eqref{eq: sde}-\eqref{eq: J0} is Markovian with respect to the augmented state $(X_s, B_{0,s}, B_{1,s})$~\footnote{We have omitted the superscript $\pi$ on every state variable for notational simplicity.}, where $B_{0,s} = b_{0} + b_{1} \int_t^s e^{-\delta (u-t)} r(u, X_u, a_u) \rd u$ and $B_{1,s} = b_{1} e^{-\delta(s-t)}$ for some $(b_0, b_1)\in \R\times \R_+$; i.e., $\pi_0^*(\cdot\mid \gF_s^X) = \pi_0^*(\cdot\mid X_s, B_{0,s}, B_{1,s})$ for any $s\in [t, T]$.
\end{proposition}

\subsection{Meta-algorithm for policy optimization}

In fact, Proposition \ref{prop: markovian policy} is derived by bridging two value functions, \eqref{eq: J0} and \eqref{eq: J} (see Appendix \ref{app: equiv opt val policy}).  Following this idea, we can design a meta-algorithm that first solves policy optimization for the augmented SDE \eqref{eq: augmented sde}-\eqref{eq: J}, and then optimizes over a scalar parameter according to the definition of the OCE risk measure $U$ to obtain the optimal value function $J_0^*$ and the optimal policy $\pi_0^*$. The meta-algorithm is summarized in Algorithm \ref{alg: meta alg}.

\section{Risk-Sensitive Q-Learning}
\label{sec: rsql}

In this section, we focus on developing a q-learning algorithm to solve for the optimal value function $J^{*}(t,x,b_0,b_1)$ and its optimal policy $\pi^*(\cdot\mid t,x,b_0,b_1)$, following the martingale approach proposed in \cite{jia2022policy1}.

\subsection{Martingale characterization}

Let $\{(M_s^\theta)_{s\in[t,T]}\}_{\theta\in\Theta}$ be a set of parameterized $\gF$-adapted stochastic processes, and let $\theta^*\in \Theta$ be the unique parameter such that $(M_s^{\theta^*})_{s\in[t,T]}$ is an $(\gF, \PB)$-martingale.
According to \cite{jia2022policy1,jia2022policy2}, $\theta=\theta^*$ if and only if for any $\gF$-adapted test process $(\xi_s)_{s\in[t,T]}$, the following \textit{martingale orthogonality condition} holds:
\begin{align}
\label{eq: martingale ortho cond}
    \E\left[ \int_t^T \xi_s \rd M_s^\theta \right] = 0.
\end{align}
This motivates us to find a martingale characterization of the value function, so that under proper parameterization, we can obtain a learned optimal value function as long as the condition \eqref{eq: martingale ortho cond} is satisfied. Theorem \ref{thm: J q martingale} below provides a rigorous statement of such martingale characterization.

\begin{theorem}
\label{thm: J q martingale}
    Let a policy $\pi$, a function $\hat{J} \in C^{1,2}([0,T) \times \R^d \times \R \times \R_+) \cap C([0,T] \times \R^d \times \R \times \R_+)$, and a continuous function $\hat{q}\colon [0, T] \times \R^d \times \R \times \R_+ \times \gA \rightarrow \R$ be given such that for any quadruple $(t,x,b_0,b_1) \in [0,T] \times \R^d \times \R \times \R_+$,
    \begin{align}
    \label{eq: constraints}
        \hat{J}(T, x, b_0, b_1) = \varphi(b_0 + b_1 h(x)),\quad \int_{\gA} \exp\left\{ \frac{\hat{q}(t, x, b_0, b_1, a)}{\tau b_1} \right\} \rd a = 1.
    \end{align}
    Then, $\hat{J} = J^{*}$ and $\hat{q} = q^{*}$ respectively, if and only if for any $(t,x) \in [0,T] \times \R^d$, the following process is an $(\gF, \PB)$-martingale:
    \begin{align}
    \label{eq: j q martingale}
        \hat{J}(s, X_s^{\pi}, Y_{s}^{\pi}, e^{-\delta (s-t)}) - \int_t^s \hat{q}(u, X_u^{\pi}, Y_{u}^{\pi}, e^{-\delta (u-t)}, a_u^{\pi}) \rd u, \quad s\in [t,T].
    \end{align}
    Here, $J^*$ is the optimal value function, $q^*$ is the optimal q-function as defined in Appendix \ref{sec: lowercase q}, and $Y_s^\pi = \int_t^s e^{-\delta(u-t)} r(u, X_u^\pi, a_u^\pi) \rd u$ is the discounted cumulative reward up to time $s$.
\end{theorem}

\subsection{Algorithm}

Based on the martingale characterization established in the previous section, we propose an on-policy continuous-time risk-sensitive q-learning (CT-RS-q) algorithm as Algorithm \ref{alg: on-policy opt}. Specifically, we simultaneously parameterize the value function as $J^\theta$ and the q-function as $q^\psi$, and set the test functions $\xi_t$ and $\zeta_t$ as their parameter gradients. The parameters $\theta$ and $\psi$ are then updated by the average of temporal-difference errors after generation of every whole episode.

\section{Application to Dynamic Portfolio Selection}
\label{sec: app mv_portfolio}

We discuss an application to dynamic portfolio selection to illustrate the effectiveness of our proposed algorithm. Consider a market with two risky assets, whose prices follow the log-normal dynamics:
\begin{align}
    \rd S_{1,t} = S_{1,t} (r_1 \rd t + \sigma_1 \rd W_{1,t}), \quad \rd S_{2,t} = S_{2,t} (r_2 \rd t + \sigma_2 \rd W_{2,t}).
\end{align}
We would like to invest a $\$ 1$ budget on these two assets and are allowed to continuously reallocate the money between assets. Denoting the proportion of money invested on the first asset as $a_t$, our budget $X_t$ follows another log-normal process:
\begin{align}
\label{eq: portfolio sde}
    \rd X_t = X_t \{a_t r_1 + (1-a_t) r_2\} \rd t + X_t \{a_t\sigma_1 \rd W_{1,t} + (1 - a_t) \sigma_2 \rd W_{2,t} \}.
\end{align}

We are concerned with our budget $X_T$ at the end of the whole period. While expecting the mean of $X_T$ to be as large as possible, we would also like to avoid excessive variance; and therefore, we choose the mean-variance risk measure as our objective:
\begin{align}
\label{eq: mvo}
    \MV(X_T) = \E[X_T] - \frac{\alpha}{2} \Var(X_T).
\end{align}
Since $\MV(\cdot)$ is an OCE as shown in Table \ref{tab: OCE}, our algorithm is applicable to the above financial scenario. A rigorous formulation and some detailed discussion are deferred to Appendix \ref{app: detail dps}.

Firstly, we examine the convergence of the model parameters as defined in Appendix \ref{app: mv_portfolio parameterization}. Figure \ref{fig: experiment converge} illustrates the evolution of eight model parameters, most of which converge to their optimal points. The last two parameters, which belong to the parameterized q-function $q^\psi$, stay slightly far from their optimal points. We believe that such deviation is caused by the non-shrinking exploration parameter $\tau>0$ in the training phase.

Secondly, we compare the performance of three policies: (i) Baseline Policy, which always invests a fixed proportion ($a=0.5$) of the budget between two assets; (ii) CT-RS-q Policy, which is trained according to Algorithm \ref{alg: on-policy opt}; (iii) Optimal Policy, whose analytical formula is given in Appendix \ref{app: mv_portfolio analytical}. Table \ref{tab: terminal performance compare} lists the cumulative return and the mean-variance objective \eqref{eq: mvo} of the three policies at the end of the whole period, and Figure \ref{fig: experiment cum_ret mvo} plots their curves over time.

We find that the trained CT-RS-q policy is close to the optimal policy, both outperforming the baseline policy in the cumulative return and the mean-variance objective. This comes at the cost of a slightly larger volatility, which is further controllable by tuning the regularization parameter $\alpha>0$ in \eqref{eq: mvo} during the training period.

\begin{table}[ht]
    \centering
    \caption{Performance comparison of three policies.}
    \label{tab: terminal performance compare}
    \begin{tabular}{lccc}
    \toprule
     & Cumulative Return $\uparrow$ (Std. Dev. $\downarrow$) & Mean-Variance Objective $\uparrow$ \\
    \midrule
    Baseline Policy ($a = 0.5$) & 0.2217 (0.0957) & 1.2171 \\
    CT-RS-q Policy & 0.8163 (0.8716) & 1.4365 \\
    Optimal Policy & 0.7128 (0.7205) & 1.4532 \\
    \bottomrule
    \end{tabular}
\end{table}



\clearpage
\bibliographystyle{apalike}
\bibliography{ref.bib}


\clearpage
\begin{sidewaystable}[htbp]
    \centering
    \begin{threeparttable}
        \resizebox{0.8\textwidth}{!}{
            \begin{tabular}{lllll}
                \toprule
                 & & \multicolumn{1}{c}{Basic Function} & \multicolumn{1}{c}{Utility Function} & \multicolumn{1}{c}{OCE Formula} \\
                 & & \multicolumn{1}{c}{$u(w)$} & \multicolumn{1}{c}{$\varphi(t)$} & \multicolumn{1}{c}{$\OCE_{\varphi}(W)$} \\
                \midrule
                \multirow{4}{*}{Risk Aversion} & General & $u(w)$ concave, increasing & $\inf_{\delta>0} \left\{ \frac{u(\delta+t) - u(\delta)}{u'(\delta)} \right\}$ & $u^{-1}(\E[u(W)])$ \\
                 & Exponential Utility & $-e^{-\alpha w}, \alpha>0$ & $\frac{1-e^{-\alpha t}}{\alpha}$ & $-\frac{1}{\alpha}\ln \E[e^{-\alpha W}]$ \\
                 & Power Utility & $\frac{w^{1-\gamma}-1}{1-\gamma}, 0<\gamma<1$ & $\inf_{\delta>\max\{0, -t\}} \left\{ \delta^\gamma \frac{(\delta+t)^{1-\gamma} - \delta^{1-\gamma}}{1-\gamma} \right\}$ & $(\E[W^{1-\gamma}])^{\frac{1}{1-\gamma}}$ \\
                 & Logarithm Utility & $\ln w$ & $\inf_{\delta>\max\{0, -t\}} \left\{ \delta \ln\left( 1+\frac{t}{\delta} \right) \right\}$ & $\exp(\E[\ln W])$ \\
                \midrule
                Risk Neutral & Linear Utility & $w$ & $t$ & $\E[W]$ \\
                \midrule
                \multicolumn{2}{l}{Conditional Value-at-Risk (CVaR)} & - & $\frac{1}{\beta}\min\{0, t\}$ & $\E[W \mid W\leq \operatorname{VaR}_{\beta}(W)]$ \\
                \multicolumn{2}{l}{Mean-Variance (MV)} & - & $t - \frac{\beta}{2}t^2$ & $\E[W] - \frac{\beta}{2}\Var(W)$ \\
                \multicolumn{2}{l}{Monotone Mean-Variance (MMV)} & - & $\min\{1, t\} - \frac{\beta}{2}(\max\{1, t\})^2$ & \textit{no explicit form} \\
                \bottomrule
            \end{tabular}
        }
        \caption{Examples of OCE risk measures and their corresponding utility functions.}
        \label{tab: OCE}
    \end{threeparttable}
\end{sidewaystable}

\clearpage
\begin{algorithm}[ht]
    \caption{Meta-algorithm}
    \label{alg: meta alg}
    \textbf{Train} based on the augmented SDE:\\
        \ (i) Optimal value function $\hat{J}^*(t,x,b_0,b_1)$\;
        (ii) Optimal policy $\hat{\pi}^*(t,x,b_0,b_1)$\;
    \SetKwFunction{FMain}{OptimalValueAndPolicy}
    \SetKwProg{Fn}{Function}{:}{}
    \Fn{\FMain{$t$, $x$, $\hat{J}^*$, $\hat{\pi}^*$}}{
        Optimal initial budget: $\hat{b}^* \leftarrow \argmax_{b \in \R} \big\{ b + \hat{J}^*(t, x, -b, 1) \big\}$\;
        Optimal value function: $\hat{J}^*_0(t,x) \leftarrow \hat{b}^* + \hat{J}^*(t, x, -\hat{b}^*, 1)$\;
        Cumulative reward: $Y_s \leftarrow \int_t^s e^{-\delta(u-t)} r(u, X_u, a_u) \rd u$\;
        Optimal policy: $\hat{\pi}^*_0(\cdot\mid s, X_s, Y_{s}) \leftarrow \hat{\pi}^*(\cdot \mid s, X_s, Y_{s}-\hat{b}^*, e^{-\delta (s-t)}), \forall s\in [t,T]$\;
        \Return{$\hat{J}^*_0$, $\hat{\pi}^*_0$}
    }
\end{algorithm}

\begin{algorithm}[ht]
    \caption{CT-RS-q: continuous-time risk-sensitive q-learning (on-policy)}
    \label{alg: on-policy opt}
    \KwIn{initial state $x_0$, number of episodes $N$, time horizon $T$, number of mesh grids $K$, \\ \qquad\quad mesh grids $0=t_0<t_1<\dots<t_K=T$, learning rates $\{l^\theta_j, l^\psi_j\}_{j=1}^N$, temperature $\tau>0$, \\ \qquad\quad parameterized value function $J^\theta(\cdot,\cdot,\cdot,\cdot)$ and q-function $q^\psi(\cdot,\cdot,\cdot,\cdot,\cdot)$}
    \For{episode $j = 1$ \KwTo $N$}{
        Observe initial state $x_{0}$ and set $(X_{t_0}, B_{0, t_0}, B_{1, t_0}) = (x_0, 0, 1)$\;
        \For{timestep $k = 0$ \KwTo $K-1$}{
            Generate action $a_{t_k}\sim \pi^{\psi}(\cdot\mid t_k, X_{t_k}, B_{0,t_k}, B_{1,t_k}) \propto \exp\{ \frac{1}{\tau B_{1,t_k}} q^\psi(t_k, X_{t_k}, B_{0,t_k}, B_{1,t_k}, \cdot) \}$\;
            Simulate \eqref{eq: augmented sde} from $t_k$ to $t_{k+1}$ and observe new state $(X_{t_{k+1}}, B_{0, t_{k+1}}, B_{1, t_{k+1}})$\;
            Store test functions: 
            \begin{align*}
                \xi_{t_k}=\frac{\partial J^\theta}{\partial \theta}(t_k, X_{t_k}, B_{0,t_k}, B_{1,t_k}),\quad \zeta_{t_k}=\frac{\partial q^\psi}{\partial \psi}(t_k, X_{t_k}, B_{0,t_k}, B_{1,t_k}, a_{t_k});
            \end{align*}\\
            Store value function and q-function: 
            \begin{align*}
                J^\theta_{t_k}=J^\theta(t_k, X_{t_k}, B_{0,t_k}, B_{1,t_k}),\quad q^\psi_{t_k} = q^\psi(t_k, X_{t_k}, B_{0,t_k}, B_{1,t_k}, a_{t_k});
            \end{align*}\\
        }
        Compute incremental updates:
        \begin{align*}
            \Delta \theta = \sum_{k=0}^{K-1} \xi_{t_k} \big[ J^{\theta}_{t_{k+1}} - J^{\theta}_{t_k} - q^{\psi}_{t_k} \Delta t_k \big],\quad \Delta \psi = \sum_{k=0}^{K-1} \zeta_{t_k} \big[ J^{\theta}_{t_{k+1}} - J^{\theta}_{t_k} - q^{\psi}_{t_k} \Delta t_k \big];
        \end{align*}\\
        Update $\theta$ and $\psi$:
        \begin{align*}
            \theta \gets \theta + l_j^{\theta} \Delta \theta, \quad \psi \gets \psi + l_j^\psi \Delta \psi;
        \end{align*}
    }
    \Return{$J^\theta$, $q^\psi$}
\end{algorithm}

\begin{figure}[ht]
    \centering
    \includegraphics[width=1.0\linewidth]{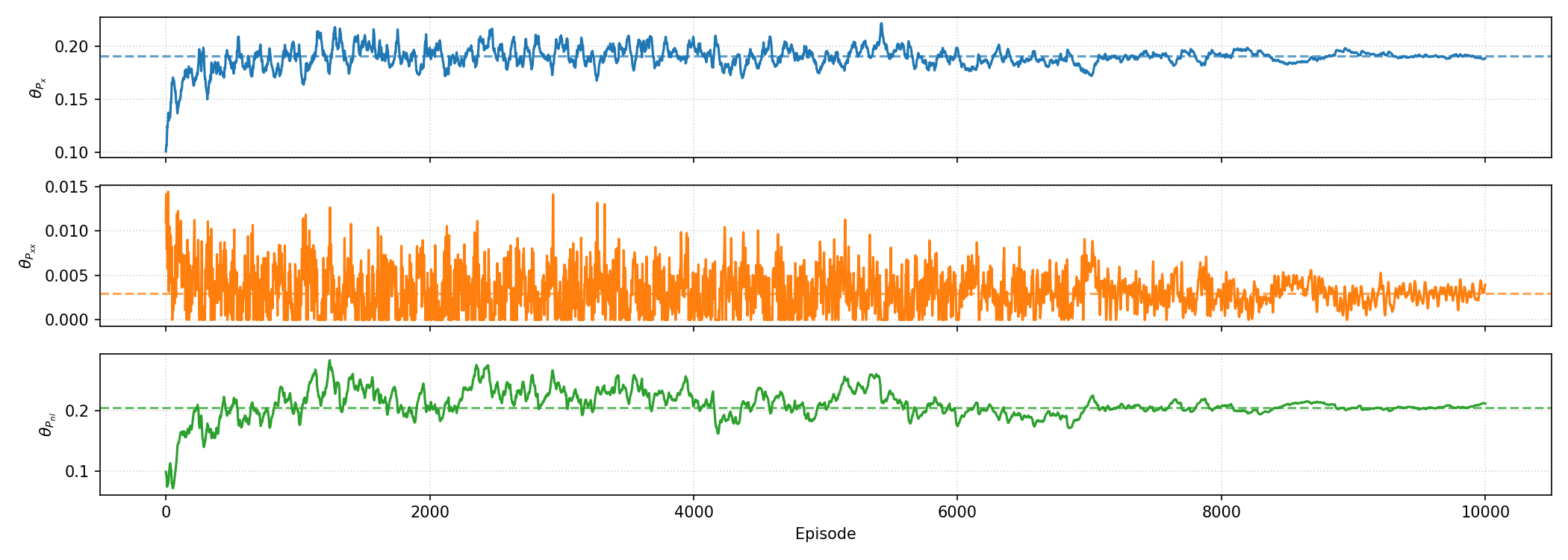}
    \includegraphics[width=1.0\linewidth]{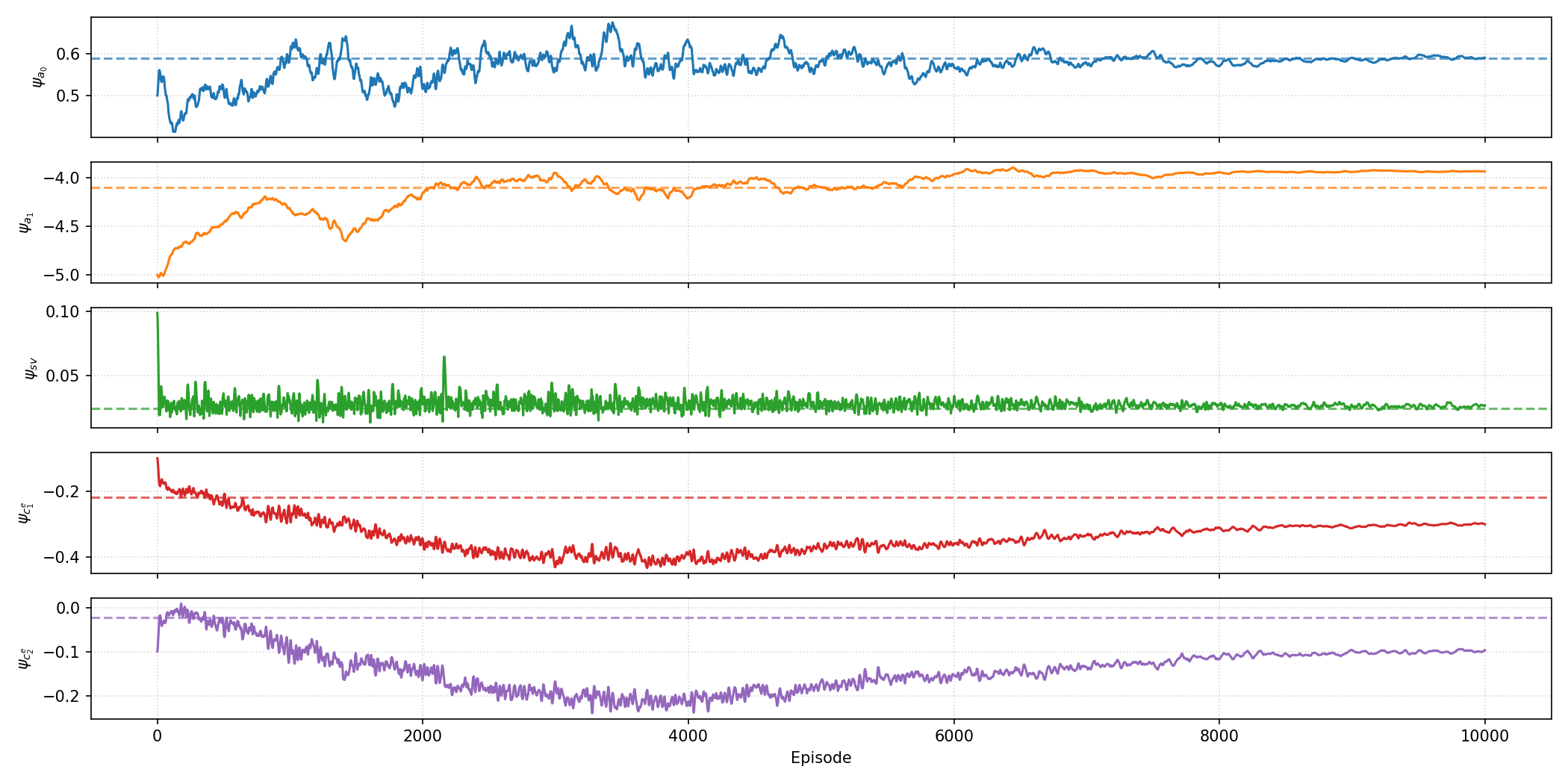}
    \caption{Convergence of model parameters. The first three are parameters of the value function $J^\theta$, and the last five are parameters of the q-function $q^\psi$.}
    \label{fig: experiment converge}
\end{figure}

\begin{figure}[ht]
    \centering
    \includegraphics[width=0.92\linewidth]{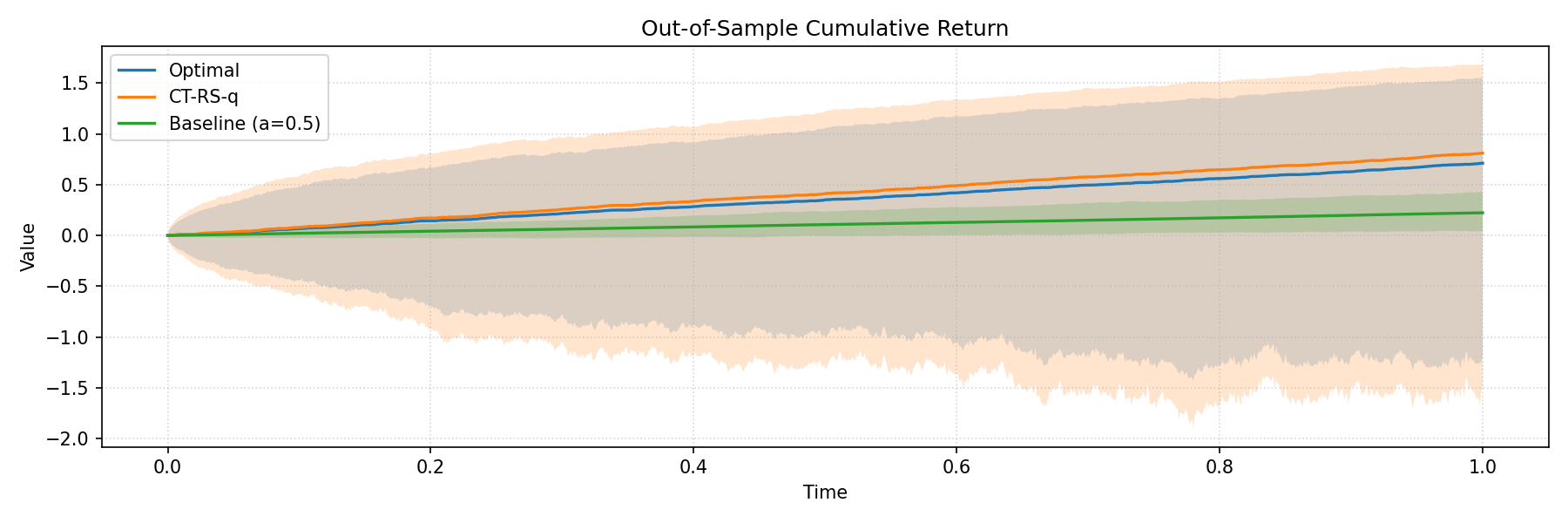}
    \includegraphics[width=0.92\linewidth]{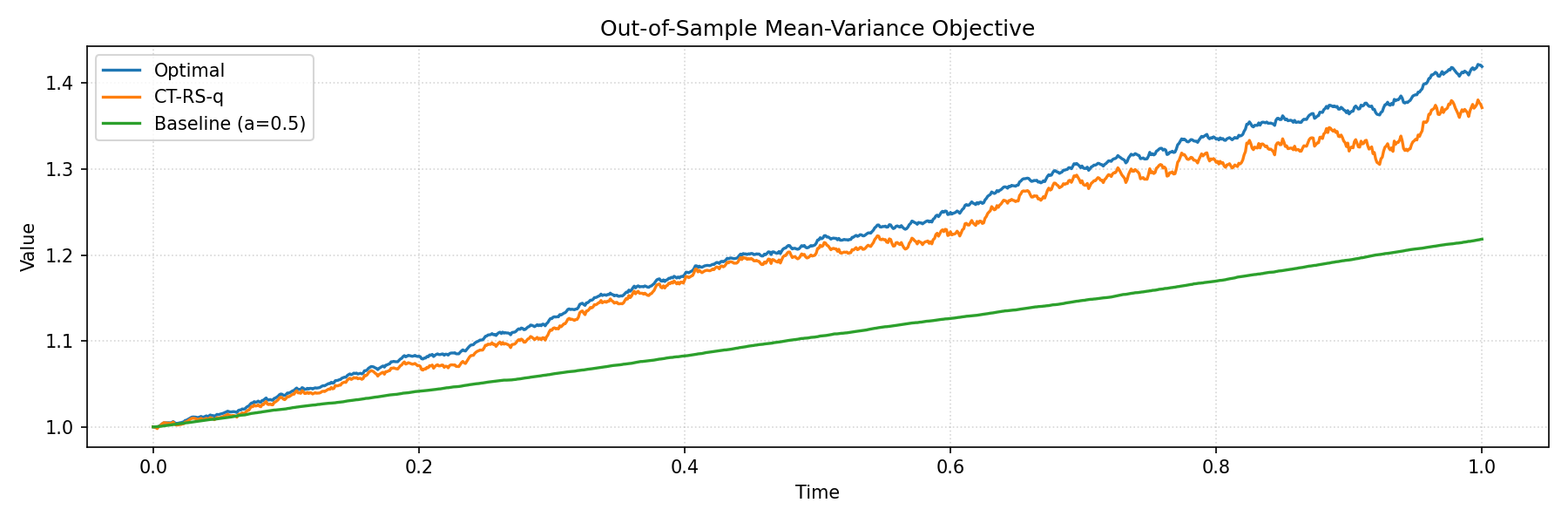}
    \caption{Curves of the cumulative return and the mean-variance objective for three policies.}
    \label{fig: experiment cum_ret mvo}
\end{figure}

\clearpage
\appendix

\section{Theory of Risk-Sensitive Q-Learning}
\label{app: theory rsql}

In this section, we provide a theoretical foundation of risk-sensitive q-learning and present more comprehensive results in addition to Section \ref{sec: rsql}.

\paragraph{Notation.} For the augmented SDE, we denote $\Pi_{\HD}$ and $\Pi_{\mkv}$ as the set of all history-dependent policies and all Markov policies, respectively. For a policy $\pi\colon \gX\times \gA \rightarrow [0,1]$ and a function $f\colon \gX\times \gA \rightarrow \R$, we use $f(\cdot, \pi) = \int_\gA f(\cdot, a) \pi(a\mid \cdot) \rd a$ to denote the function value averaged over the policy.

\subsection{Feynman-Kac formula and HJB equation}

Let $\gL_{\aug}^a$ be the infinitesimal generator associated with the diffusion process governed by \eqref{eq: augmented sde}:
\begin{equation}
    \begin{aligned}
        \gL_{\aug}^a f(t,x,b_0,b_1) &= \partial_t f(t,x,b_0,b_1) + \mu_{\aug}(t, x, b_0, b_1, a) \partial_x f(t,x,b_0,b_1) \\
        &\quad + \frac{1}{2}\sigma_{\aug}^2(t,x,b_0,b_1,a) \partial^2_{xx} f(t,x,b_0,b_1).
    \end{aligned}
\end{equation}
Recall that the augmented value function can be written as follows:
\begin{equation}
    \begin{aligned}
        J^\pi(t, x, b_0, b_1) &= \E \bigg[\varphi\big( B_{0,T}^\pi + B_{1,T}^\pi h(X_T^\pi) \big) \\
        &\qquad\quad - \tau \int_t^T B_{1,s}^\pi \log \pi(a_s^\pi\mid \gF_s^X ) \rd s \mid X_t^\pi=x, B_{0,t}^\pi=b_0, B_{1,t}^\pi=b_1 \bigg].
    \end{aligned}
\end{equation}

Given a Markov policy $\pi\in \Pi_{\mkv}$, its augmented value function $J^\pi$ satisfies the following PDE:
\begin{equation}
\label{eq: J feynman kac}
    \begin{aligned}
        \int_\gA \big[ \gL_{\aug}^a J^\pi( t,x,b_0,b_1) - \tau b_1 \log \pi(a\mid t,x,b_0,b_1) \big] & \pi(a\mid t,x,b_0,b_1) \rd a = 0, \\
        J^\pi(T,x,b_0,b_1) = \varphi(b_0 + b_1 h(x)) &,
    \end{aligned}
\end{equation}
which is the well-known Feynman-Kac formula in the exploratory RL setting.

On the other hand, the optimal value function $J^*$ satisfies the following HJB equation:
\begin{equation}
    \begin{aligned}
        \sup_{\pi\in \Pi_{\mkv}} \int_\gA \big[ \gL_{\aug}^a J^*( t,x,b_0,b_1) - \tau b_1 \log \pi(a\mid t,x,b_0,b_1) \big] &\pi(a\mid t,x,b_0,b_1) \rd a = 0, \\
        J^*(T,x,b_0,b_1) = \varphi(b_0 + b_1 h(x)) &,
    \end{aligned}
\end{equation}
Solving the above HJB equation yields the relationship between the optimal value function $J^*$ and the optimal policy $\pi^*$, which we state in the following proposition.

\begin{proposition}
    The optimal value function $J^*$ and the optimal policy $\pi^*$ satisfy the following relationship:
    \begin{align}
        \pi^*(a\mid t,x,b_0,b_1) = \exp\left\{ \frac{\gL_{\aug}^a J^*(t,x,b_0,b_1)}{\tau b_1} \right\},\quad \int_{\gA} \exp\left\{ \frac{\gL_{\aug}^a J^*(t,x,b_0,b_1)}{\tau b_1} \right\} \rd a = 1.
    \end{align}
\end{proposition}

\subsection{\textit{Lowercase} q-function}
\label{sec: lowercase q}

We next focus on deriving the q-function with respect to an augmented value function $J^{\pi}(t, x, b_0, b_1)$, analogous to Q-functions in the conventional RL setting. The concept of the \textit{lowercase} q-function was proposed in \cite{jia2023q}, where the authors provided a rigorous justification for a recent conjecture \citep{gao2022state,zhou2021curse} that the counterpart of Q-functions in continuous-time RL is the Hamiltonian of the dynamics. Below we extend the theory to the risk-sensitive scenario.

Given a Markov policy $\pi\in \Pi_{\mkv}$, a fixed action $a\in \gA$ and a small constant $\Delta t>0$, consider a perturbed policy as follows: it takes the action $a$ on $[t, t+\Delta t)$, and then follows $\pi$ on $[t+\Delta t, T]$. The cumulative reward under such a perturbed policy then becomes
\begin{align}
    Z^\pi_{\Delta t}(t,x, a) &= \int_t^{t+\Delta t} e^{-\delta (s-t)} r(s, X_s^{a}, a) \rd s
     + e^{-\delta \Delta t} Z^{\pi}(t+\Delta t, X_{t+\Delta t}^{a}),\quad X_t^a = x,
\end{align}
and we introduce the corresponding $\Delta t$-parameterized Q-function as
\begin{align}
    Q_{\Delta t}^{\pi}(t, x, b_0, b_1, a) = \E[ \varphi( b_0 + b_1 Z^\pi_{\Delta t}(t,x, a) )] + \tau b_1 e^{-\delta \Delta t} \E[\Ent(\pi; t+\Delta t, X_{t+\Delta t}^{\pi})].
\end{align}

Recall that in \eqref{eq: J0} and \eqref{eq: J} an entropy term is included to incentivize exploration using stochastic policies. However, in defining $Q_{\Delta t}^{\pi}(t,x,b_0,b_1,a)$ we exclude the policy's entropy on the interval $[t, t+\Delta t)$, because a deterministic constant action $a$ is applied whose entropy is always zero.

It is obvious that when $\Delta t \rightarrow 0$, the $\Delta t$-parameterized Q-function converges to the value function $J^{\pi}(t,x,b_0,b_1)$. However, the first-order term of $\Delta t$ crucially reflects the advantage of the action $a$ over the current policy $\pi$, which shares the same meaning as Q-functions in discrete-time RL. Following previous works \citep{jia2023q}, we define such a first-order term as the q-function. To proceed, we define the infinitesimal generator of the SDE \eqref{eq: sde} as $\gL^a$:
\begin{align}
    \gL^a f(t,x) = \partial_t f(t,x) + \mu(t, x, a) \partial_x f(t,x) + \frac{1}{2}\sigma^2(t,x,a) \partial^2_{xx} f(t,x).
\end{align}

\begin{definition}
\label{def: q-function}
    The \textit{q-function} of the augmented SDE \eqref{eq: augmented sde}-\eqref{eq: J} associated with a policy $\pi\in \Pi_{\mkv}$ is defined as follows:
    \begin{align}
    \label{eq: q-function}
        q^{\pi}(t, x, b_0, b_1, a) &= \{ \gL^a J^{\pi} + b_1 r(t,x,a) \partial_{b_0} J^\pi - b_1 \delta \partial_{b_1} J^\pi \}(t, x, b_0, b_1).
    \end{align}
\end{definition}

Note that any q-function $q^\pi$ is related with the value function $J^\pi$ of the same policy. In addition, we define the optimal q-function as $q^*(t,x,b_0,b_1,a) = q^{\pi^*}(t,x,b_0,b_1,a)$. Below we present some important properties of the q-function.

\begin{proposition}
\label{prop: q-function}
    The $\Delta t$-parameterized Q-function $Q_{\Delta t}^{\pi}$ satisfies that
    \begin{align}
        Q^{\pi}_{\Delta t}(t, x, b_0, b_1, a) = J^{\pi}(t, x, b_0, b_1) + q^{\pi}(t, x, b_0, b_1, a) \Delta t + o(\Delta t).
    \end{align}
\end{proposition}

\begin{proposition}
\label{prop: q-function mean}
    The q-function $q^\pi$ satisfies that
    \begin{align}
        \int_\gA \{ q^\pi(t,x,b_0,b_1,a) - \tau b_1  \log \pi(a\mid t,x,b_0,b_1) \} \pi(a\mid t,x,b_0,b_1)\rd a = 0.
    \end{align}
\end{proposition}

\begin{proposition}
\label{prop: equiv q L_aug}
    The q-function $q^\pi$ and the value function $J^\pi$ satisfy that
    \begin{align}
        q^\pi(t,x,b_0,b_1,a) = \gL_{\aug}^a J^\pi(t,x,b_0,b_1).
    \end{align}
\end{proposition}

\subsection{Martingale characterization}

The definition of the q-function enables us to design various kinds of martingale characterization of the value function, based on different algorithmic requirements. Specifically, we present below the martingale characterization theorems for: (i) on-policy policy evaluation in Theorem \ref{thm: pe J q martingale}, (ii) off-policy policy evaluation in Theorems \ref{thm: q pi pi' martingale} and \ref{thm: J q pi pi' martingale}, and (iii) off-policy policy optimization in Theorem \ref{thm: optimal q function policy}. Their proofs are similar to that of Theorem \ref{thm: J q martingale}, so we omit them for brevity. As before, we define the discounted cumulative reward up to time $s$ as $Y_s^\pi$:
\begin{align}
    Y_s^\pi = \int_t^s e^{-\delta(u-t)} r(u, X_u^\pi, a_u^\pi) \rd u, \quad s\in [t,T].
\end{align}

\begin{theorem}
\label{thm: pe J q martingale}
    Let a policy $\pi\in \Pi_{\mkv}$, its corresponding value function $J^{\pi}$ and q-function $q^{\pi}$, a function $\hat{J} \in C^{1,2}([0,T) \times \R^d \times \R \times \R_+) \cap C([0,T] \times \R^d \times \R \times \R_+)$ with polynomial growth, and a continuous function $\hat{q}\colon [0, T] \times \R^d \times \R \times \R_+ \times \gA \rightarrow \R$ be given such that for any $(t,x,b_0,b_1) \in [0,T] \times \R^d \times \R \times \R_+$,
    \begin{align}
    \label{eq: pe constraints}
        \hat{J}(T, x, b_0, b_1) = \varphi(b_0 + b_1 h(x)),\quad \hat{q}(t, x, b_0, b_1, \pi) + \tau b_1 \Ent(\pi(\cdot\mid t, x, b_0, b_1)) = 0.
    \end{align}
    Then:
    \begin{enumerate}[(i)]
        \item $\hat{q} = q^{\pi}$ if and only if for any $(t, x) \in [0, T] \times \R^d$, the following process is an $(\gF, \PB)$-martingale:
        \begin{align}
            J^{\pi}(s, X_s^{\pi}, Y_s^{\pi}, e^{-\delta (s-t)}) - \int_t^s \hat{q}(u, X_u^{\pi}, Y_u^{\pi}, e^{-\delta (u-t)}, a_u^{\pi}) \rd u, \quad s\in [t,T].
        \end{align}
        \item $\hat{J} = J^{\pi}$ and $\hat{q} = q^{\pi}$ respectively, if and only if for any $(t,x) \in [0,T] \times \R^d$, the following process is an $(\gF, \PB)$-martingale:
        \begin{align}
        \label{eq: pe j q martingale}
            \hat{J}(s, X_s^{\pi}, Y_s^{\pi}, e^{-\delta (s-t)}) - \int_t^s \hat{q}(u, X_u^{\pi}, Y_u^{\pi}, e^{-\delta (u-t)}, a_u^{\pi}) \rd u, \quad s\in [t,T].
        \end{align}
    \end{enumerate}
\end{theorem}

\begin{theorem}
\label{thm: q pi pi' martingale}
    Let a policy $\pi\in \Pi_{\mkv}$, its corresponding value function $J^{\pi}$ and q-function $q^{\pi}$, and a continuous function $\hat{q}\colon [0, T] \times \R^d \times \R \times \R_+ \times \gA \rightarrow \R$ be given. Then:
    \begin{enumerate}[(i)]
        \item If $\hat{q} = q^{\pi}$, then for any $\pi'\in \Pi_{\mkv}$ and any $(t, x) \in [0, T] \times \R^d$, the following process is an $(\gF, \PB)$-martingale:
        \begin{align}
        \label{eq: q pi pi' martingale}
            J^{\pi}(s, X_s^{\pi'}, Y_s^{\pi'}, e^{-\delta (s-t)}) - \int_t^s \hat{q}(u, X_u^{\pi'}, Y_u^{\pi'}, e^{-\delta (u-t)}, a_u^{\pi'}) \rd u, \quad s\in [t,T].
        \end{align}
        \item If there exists $\pi'\in \Pi_{\mkv}$ such that \eqref{eq: q pi pi' martingale} is an $(\gF, \PB)$-martingale for any $(t, x) \in [0, T] \times \R^d$, then $\hat{q} = q^{\pi}$.
    \end{enumerate}
\end{theorem}

\begin{theorem}
\label{thm: J q pi pi' martingale}
    Let a policy $\pi\in \Pi_{\mkv}$, its corresponding value function $J^{\pi}$ and q-function $q^{\pi}$, a function $\hat{J} \in C^{1,2}([0,T) \times \R^d \times \R \times \R_+) \cap C([0,T] \times \R^d \times \R \times \R_+)$ with polynomial growth, and a continuous function $\hat{q}\colon [0, T] \times \R^d \times \R \times \R_+ \times \gA \rightarrow \R$ be given such that for any $(t,x,b_0,b_1) \in [0,T] \times \R^d \times \R \times \R_+$,
    \begin{align}
    \label{eq: J q pi pi' constraint}
        \hat{J}(T, x, b_0, b_1) = \varphi(b_0 + b_1 h(x)),\quad \hat{q}(t, x, b_0, b_1, \pi) + \tau b_1 \Ent(\pi(\cdot\mid t, x, b_0, b_1)) = 0.
    \end{align}
    Then:
    \begin{enumerate}[(i)]
        \item If $\hat{J} = J^{\pi}$ and $\hat{q} = q^{\pi}$ respectively, then for any $\pi'\in \Pi_{\mkv}$ and any $(t, x) \in [0, T] \times \R^d$, the following process is an $(\gF, \PB)$-martingale:
        \begin{align}
        \label{eq: J q pi pi' martingale}
            \hat{J}(s, X_s^{\pi'}, Y_s^{\pi'}, e^{-\delta (s-t)}) - \int_t^s \hat{q}(u, X_u^{\pi'}, Y_u^{\pi'}, e^{-\delta (u-t)}, a_u^{\pi'}) \rd u, \quad s\in [t,T].
        \end{align}
        \item If there exists $\pi'\in \Pi_{\mkv}$ such that \eqref{eq: J q pi pi' martingale} is an $(\gF, \PB)$-martingale for any $(t, x) \in [0, T] \times \R^d$, then $\hat{J} = J^{\pi}$ and $\hat{q} = q^{\pi}$ respectively.
    \end{enumerate}
\end{theorem}

\begin{theorem}
\label{thm: optimal q function policy}
    Let a function $\hat{J}^* \in C^{1,2}([0, T) \times \R^d \times \R \times \R_+) \cap C([0, T] \times \R^d \times \R \times \R_+)$ with polynomial growth and a continuous function $\hat{q}^* \colon [0,T]\times \R^d \times \R \times \R_+ \times \gA \rightarrow \R$ be given such that for any $(t, x, b_0, b_1) \in [0, T] \times \R^d \times \R \times \R_+$,
    \begin{align}
    \label{eq: opt qv constraints}
        \hat{J}^*(T, x, b_0, b_1) = \varphi(b_0 + b_1 h(x)),\quad \int_{\gA} \exp\left\{ \frac{\hat{q}^*(t,x,b_0,b_1,a)}{\tau b_1} \right\} \rd a = 1.
    \end{align}
    Then:
    \begin{enumerate}[(i)]
        \item If $\hat{J}^*$ and $\hat{q}^*$ are respectively the optimal value function and the optimal q-function, then for any $\pi\in \Pi_{\mkv}$ and any $(t,x) \in [0,T] \times \R^d$, the following process is an $(\gF, \PB)$-martingale:
        \begin{align}
        \label{eq: optimal value policy martingale}
            \hat{J}^{*}(s, X_s^{\pi}, Y_s^{\pi}, e^{-\delta (s-t)}) - \int_t^s \hat{q}^*(u, X_u^{\pi}, Y_u^{\pi}, e^{-\delta (u-t)}, a_u^{\pi}) \rd u, \quad s\in [t,T].
        \end{align}
        Moreover, $\hat{\pi}^*(a\mid t, x, b_0, b_1) = \exp \left\{ \frac{\hat{q}^*(t,x,b_0,b_1,a)}{\tau b_1} \right\}$ is the optimal policy in this case.
        \item If there exists $\pi\in \Pi_{\mkv}$ such that for any $(t,x) \in [0,T] \times \R^d$, \eqref{eq: optimal value policy martingale} is an $(\gF, \PB)$-martingale, then $\hat{J}^*$ and $\hat{q}^*$ are respectively the optimal value function and the optimal q-function.
    \end{enumerate}
\end{theorem}

\clearpage
\section{Details on Dynamic Portfolio Selection}
\label{app: detail dps}

In this section, we provide a more comprehensive introduction to dynamic portfolio selection. Consider a market with two risky assets, with their prices following the log-normal dynamics:
\begin{align}
    \rd S_{1,t} = S_{1,t} (r_1 \rd t + \sigma_1 \rd W_{1,t}), \quad \rd S_{2,t} = S_{2,t} (r_2 \rd t + \sigma_2 \rd W_{2,t}),
\end{align}
where $r_1, r_2 \in \R$, $\sigma_1, \sigma_2 \in \R_+$, and $(W_{1,t})_{t\geq 0}, (W_{2,t})_{t\geq 0}$ are two independent Brownian motions. Suppose we have a $\$ 1$ budget at $t=0$ and would like to invest all the money in these two assets; at each time $t$, we are allowed to reallocate the money between assets. Denoting the proportion of money invested on the first asset as $a_t$, our budget $X_t$ follows another log-normal process:
\begin{align}
\label{eq: portfolio sde app}
    \rd X_t = X_t \{a_t r_1 + (1-a_t) r_2\} \rd t + X_t \{a_t\sigma_1 \rd W_{1,t} + (1 - a_t) \sigma_2 \rd W_{2,t} \}.
\end{align}
Here, $a_t$'s are possible to take values outside the unit interval $[0,1]$, as we allow shorting of an asset.

We are concerned with our budget $X_T$ at the end of the whole period without any discount in time; i.e., $\delta = 0$. We choose the mean-variance risk measure as our objective:
\begin{align}
    \MV(X_T) = \E[X_T] - \frac{\alpha}{2} \Var(X_T).
\end{align}
Then the corresponding utility function and reward functions are as follows:
\begin{align}
\label{eq: portfolio reward}
    \varphi(x) = x - \frac{\alpha}{2} x^2,\quad r(t,x,a) = 0,\quad h(x) = x.
\end{align}

The remainder of this section is organized as follows. Appendix \ref{app: mv_portfolio analytical} gives analytical formulae for the optimal value function and its corresponding q-function with zero exploration factor; i.e., $\tau = 0$. According to the forms of functions, we design their parameterization with well specification at optimum in Appendix \ref{app: mv_portfolio parameterization}. In Appendix \ref{app: mv_portfolio numerical}, we numerically justify that the optimal parameter in our parameterization is a stable point, which indicates local convergence from near the optimum.

\subsection{Analytical solution to optimal control}
\label{app: mv_portfolio analytical}

The following proposition gives the analytical formulae of the optimal control, the optimal value function, and its corresponding q-function. We assume zero exploration factor, $\tau=0$, for simplicity. Note that both the value function and the q-function are quadratic in $x$ and $a$.

\begin{proposition}
\label{prop: mv_portfolio optimal}
    The optimal control to \eqref{eq: portfolio sde app}-\eqref{eq: portfolio reward}, the optimal value function, and its corresponding q-function are
    \begin{align*}
        &a^*(t,x,b_0,b_1) = \frac{\sigma_2^2}{\sigma_1^2 + \sigma_2^2} - \frac{r_1 - r_2}{\sigma_1^2 + \sigma_2^2} \left( 1 + \frac{c_1}{2c_2 x} \right), \\
        &J^*(t,x,b_0,b_1) = c_0 + c_1 x + c_2 x^2, \\
        &q^*(t,x,b_0,b_1,a) = (\sigma_1^2 + \sigma_2^2) c_2 x^2\{a - a^*(t,x,b_0,b_1)\}^2,
    \end{align*}
    where $c_0,c_1,c_2$ are defined as follows:
    \begin{align*}
        &c_0 = c_0(t,b_0,b_1) = b_0\left(1 - \frac{\alpha}{2}b_0 \right) + \frac{(1-\alpha b_0)^2 P_{nl}}{2\alpha(P_{xx} + P_{nl})} \left[ 1 - e^{-2(P_{xx}+P_{nl})(T-t)} \right], \\
        &c_1 = c_1(t,b_0,b_1) = (1 - \alpha b_0) b_1 e^{(P_x - 2P_{nl})(T-t)}, \\
        &c_2 = c_2(t,b_0,b_1) = -\frac{\alpha}{2}b_1^2 e^{2(P_x+P_{xx} - P_{nl})(T-t)}, \\
        &P_x = \frac{r_1 \sigma_2^2 + r_2 \sigma_1^2}{\sigma_1^2 + \sigma_2^2},\quad P_{xx} = \frac{\sigma_1^2\sigma_2^2}{2(\sigma_1^2 + \sigma_2^2)}, \quad P_{nl} = \frac{(r_1 - r_2)^2}{2(\sigma_1^2 + \sigma_2^2)}.
    \end{align*}
\end{proposition}

\begin{proof}[Proof of Proposition \ref{prop: mv_portfolio optimal}]
    Recall the optimal value function is
    \begin{align*}
        J^*(t,x,b_0,b_1) = \sup_a \E\left[(b_0 + b_1 X_T) - \frac{\alpha}{2}(b_0 + b_1 X_T)^2 \mid X_t = x \right].
    \end{align*}
    Its HJB equation is as follows:
    \begin{align*}
        \sup_{a} \left\{ \partial_t J^* + \{ a r_1 + (1-a) r_2 \} x\partial_x J^* + \frac{1}{2} \{a^2 \sigma_1^2 + (1-a)^2 \sigma_2^2\} x^2 \partial^2_{xx} J^* \right\} = 0.
    \end{align*}
    For now we assume $\partial_{xx}^2 J^* < 0$, so the left-hand side is a quadratic function with respect to $a$ and achieves optimality at
    \begin{align*}
        a^* = \frac{\sigma_2^2 x^2 \partial^2_{xx} J^* - (r_1 - r_2) x \partial_x J^*}{(\sigma_1^2 + \sigma_2^2) x^2 \partial^2_{xx} J^*} = \frac{\sigma_2^2}{\sigma_1^2 + \sigma_2^2} - \frac{(r_1 - r_2) \partial_x J^*}{(\sigma_1^2 + \sigma_2^2) x \partial^2_{xx} J^*},
    \end{align*}
    and by plugging back into the HJB equation we obtain
    \begin{align}
    \label{eq: portolio selection hjb}
        \partial_t J^* + \underbrace{\frac{r_1\sigma_2^2 + r_2 \sigma_1^2}{\sigma_1^2 + \sigma_2^2} }_{P_x} x \partial_x J^* + \underbrace{ \frac{\sigma_1^2 \sigma_2^2}{2(\sigma_1^2 + \sigma_2^2)} }_{P_{xx}} x^2 \partial_{xx} J^* - \underbrace{ \frac{(r_1 - r_2)^2}{2(\sigma_1^2 + \sigma_2^2)} }_{P_{nl}} \frac{(\partial_x J^*)^2}{\partial_{xx} J^*} = 0.
    \end{align}
    
    Assume the value function has a quadratic form in $x$: 
    \begin{align*}
        J^*(t,x,b_0,b_1) = c_0(t,b_0,b_1) + c_1(t,b_0,b_1) x + c_2(t,b_0,b_1) x^2,
    \end{align*}
    which satisfies the terminal condition:
    \begin{align*}
        c_0(T,b_0,b_1) = b_0 - \frac{\alpha}{2}b_0^2,\quad c_1(T,b_0,b_1) = b_1 - \alpha b_0 b_1,\quad c_2(T, b_0, b_1) = -\frac{\alpha}{2} b_1^2.
    \end{align*}
    The partial derivatives with respect to $t$ and $x$ are thus given by
    \begin{align*}
        \partial_t J^* = \dot{c}_0 + \dot{c}_1 x + \dot{c}_2 x^2, \quad \partial_x J^* = c_1 + 2c_2 x, \quad \partial^2_{xx} J^* = 2c_2.
    \end{align*}
    Plug these into \eqref{eq: portolio selection hjb} and note that it holds for any $x\in \R^d$, so the following differential equations should be satisfied:
    \begin{align*}
        \dot{c}_0 - P_{nl} \frac{c_1^2}{2c_2} = 0,& \quad c_0(T,b_0,b_1) = b_0 - \frac{\alpha}{2}b_0^2, \\
        \dot{c}_1 + (P_x - 2P_{nl}) c_1 = 0,& \quad c_1(T,b_0,b_1) = b_1 - \alpha b_0 b_1, \\
        \dot{c}_2 + 2(P_x + P_{xx} - P_{nl}) c_2 = 0,& \quad c_2(T, b_0, b_1) = -\frac{\alpha}{2} b_1^2.
    \end{align*}
    Solving these equations yields
    \begin{align*}
        c_0(t,b_0,b_1) &= b_0\left(1 - \frac{\alpha}{2}b_0 \right) + \frac{(1-\alpha b_0)^2 P_{nl}}{2\alpha(P_{xx} + P_{nl})} \left[ 1 - e^{-2(P_{xx}+P_{nl})(T-t)} \right], \\
        c_1(t,b_0,b_1) &= (1 - \alpha b_0) b_1 e^{(P_x - 2P_{nl})(T-t)}, \\
        c_2(t,b_0,b_1) &= -\frac{\alpha}{2}b_1^2 e^{2(P_x+P_{xx} - P_{nl})(T-t)}.
    \end{align*}
    Note that $c_2<0$, so the previous assumption $\partial_{xx}^2 J^* < 0$ is verified. Hence, the corresponding optimal control is
    \begin{align*}
        a^*(t,x,b_0,b_1) = \frac{\sigma_2^2}{\sigma_1^2 + \sigma_2^2} - \frac{r_1 - r_2}{\sigma_1^2 + \sigma_2^2}\left( 1 + \frac{c_1}{2c_2x} \right).
    \end{align*}

    Finally, we turn to the q-function. Applying \eqref{eq: q-function} yields
    \begin{align*}
        & q^*(t,x,b_0,b_1,a) \\
        &= \gL^a J^{*}(t, x, b_0, b_1) + b_1 r(t,x,a) \partial_{b_0} J^*(t,x,b_0,b_1) - b_1 \delta \partial_{b_1} J^*(t,x,b_0,b_1) \\
        &= (\dot{c}_0 + \dot{c}_1 x + \dot{c}_2 x^2) + \{ar_1+(1-a)r_2\} x(c_1+2c_2x) + \{a^2\sigma_1^2 + (1-a)^2 \sigma_2^2\} c_2 x^2 \\
        &= (\sigma_1^2 + \sigma_2^2) c_2 x^2 a^2 + \{(r_1-r_2)(c_1x + 2c_2 x^2) - 2\sigma_2^2 c_2 x^2\} a + \text{remainder} \\
        &= (\sigma_1^2 + \sigma_2^2) c_2 x^2 \left( a + \frac{(r_1-r_2)(c_1x + 2c_2 x^2) - 2\sigma_2^2 c_2 x^2}{2(\sigma_1^2 + \sigma_2^2) c_2 x^2} \right)^2 \\
        &= (\sigma_1^2 + \sigma_2^2) c_2 x^2 \{ a - a^*(t,x,b_0,b_1)\}^2,
    \end{align*}
    where the second to last equality is because the remainder term is independent with $a$ and $\max_{a} q^*(t,x,b_0,b_1,a) = 0$.
\end{proof}

\subsection{Parameterization of trainable functions}
\label{app: mv_portfolio parameterization}

This section discusses how we parameterize the value function and the q-function.

The value function is parameterized as $J^{\theta} = c_0^{\theta} + c_1^{\theta} x + c_2^{\theta} x^2$ with the parameter $\theta = (\theta_{P_x}, \theta_{P_{xx}}, \theta_{P_{nl}})$, where $c^{\theta}_0$, $c^{\theta}_1$, $c^{\theta}_2$ are defined as follows:
\begin{align*}
    &c^{\theta}_0 = c^{\theta}_0(t,b_0,b_1) = b_0\left(1 - \frac{\alpha}{2}b_0 \right) + \frac{(1-\alpha b_0)^2 \theta_{P_{nl}}}{2\alpha(\theta_{P_{xx}} + \theta_{P_{nl}})} \left[ 1 - e^{-2(\theta_{P_{xx}}+\theta_{P_{nl}})(T-t)} \right], \\
    &c^{\theta}_1 = c^{\theta}_1(t,b_0,b_1) = (1 - \alpha b_0) b_1 e^{(\theta_{P_x} - 2\theta_{P_{nl}})(T-t)}, \\
    &c^{\theta}_2 = c^{\theta}_2(t,b_0,b_1) = -\frac{\alpha}{2}b_1^2 e^{2(\theta_{P_x}+\theta_{P_{xx}} - \theta_{P_{nl}})(T-t)}.
\end{align*}

The q-function is parameterized as $q^{\psi}(t,x,b_0,b_1,a) = \psi_{sv} c_2^{\psi} x^2 ( a - a^{\psi} )^2$ with the parameter $\psi = (\psi_{a_0}, \psi_{a_1}, \psi_{sv}, \psi_{c_1^e}, \psi_{c_2^e})$, where
\begin{align*}
    c_1^{\psi} &= c^{\psi}_1(t,b_0,b_1) = (1 - \alpha b_0) b_1 e^{\psi_{c_1^e}(T-t)}, \\
    c_2^{\psi} &= c^{\psi}_2(t,b_0,b_1) = -\frac{\alpha}{2}b_1^2 e^{\psi_{c_2^{e}}(T-t)}, \\
    a^{\psi} &= \psi_{a_0} - \psi_{a_1} \left( 1 + \frac{c_1^{\psi}}{2 c_2^{\psi} x} \right).
\end{align*}

Note that our parameterization well specifies the ground truth model when the following optimal parameters are achieved:
\begin{align*}
    &\theta^*_{P_x} = P_x,\quad \theta^*_{P_{xx}} = P_{xx},\quad \theta^*_{P_{nl}} = P_{nl}, \\
    &\psi^*_{a_0} = \frac{\sigma_2^2}{\sigma_1^2 + \sigma_2^2},\quad \psi^*_{a_1} = \frac{r_1 - r_2}{\sigma_1^2 + \sigma_2^2},\quad \psi^*_{sv} = \sigma_1^2 + \sigma_2^2, \\ 
    &\psi^*_{c_1^e} = P_x - 2P_{nl},\quad \psi^*_{c_2^e} = 2(P_x + P_{xx} - P_{nl}).
\end{align*}

\subsection{Numerical analysis of temporal difference around optimum}
\label{app: mv_portfolio numerical}

Figure \ref{fig:mv_portfolio optimal grad} shows temporal difference of parameters around the optimum. The environment parameters are set as $r_1 = 0.15$, $r_2 = 0.25$, $\sigma_1 = 0.1$, $\sigma_2 = 0.12$, and the hyper-parameters are set as $T=1.0$, $\Delta t=0.001$, $\alpha = 1.0$. We simulate $N=10,000$ trajectories starting from the initial budget $X_0 = 1.0$. In this setting, the value of optimal parameters $\theta^* = (\theta^*_{P_x}, \theta^*_{P_{xx}}, \theta^*_{P_{nl}})$ and $\psi^* = (\psi^*_{a_0}, \psi^*_{a_1}, \psi^*_{sv}, \psi^*_{c_1^e}, \psi^*_{c_2^e})$ are as follows:
\begin{align*}
    \theta^*_{P_x} = 0.1910, \quad \theta^*_{P_{xx}} &= 0.0030, \quad \theta^*_{P_{nl}} = 0.2049, \\
    \psi^*_{a_0} = 0.5902, \quad \psi^*_{a_1} = -4.0984, \quad \psi^*_{sv} &= 0.0244, \quad \psi^*_{c_1^e} = -0.2189, \quad \psi^*_{c_2^e} = -0.0220.
\end{align*}
For each parameter, it is perturbed around its optimal value by a small amount, while other parameters remain fixed. The dot at zero in each plot corresponds to the optimum, and the dashed black lines are the zero horizontal and vertical lines on which the update for the parameter is zero.

From the plots in Figure \ref{fig:mv_portfolio optimal grad} we find that, for any of the eight parameters in the value function or the q-function, its temporal difference crosses zero from above to below, which indicates the optimum is a stable point. This phenomenon numerically provides guarantees for convergence to optimum of the given optimization algorithm.

\begin{figure}[ht]
    \centering
    \includegraphics[width=1.0\linewidth]{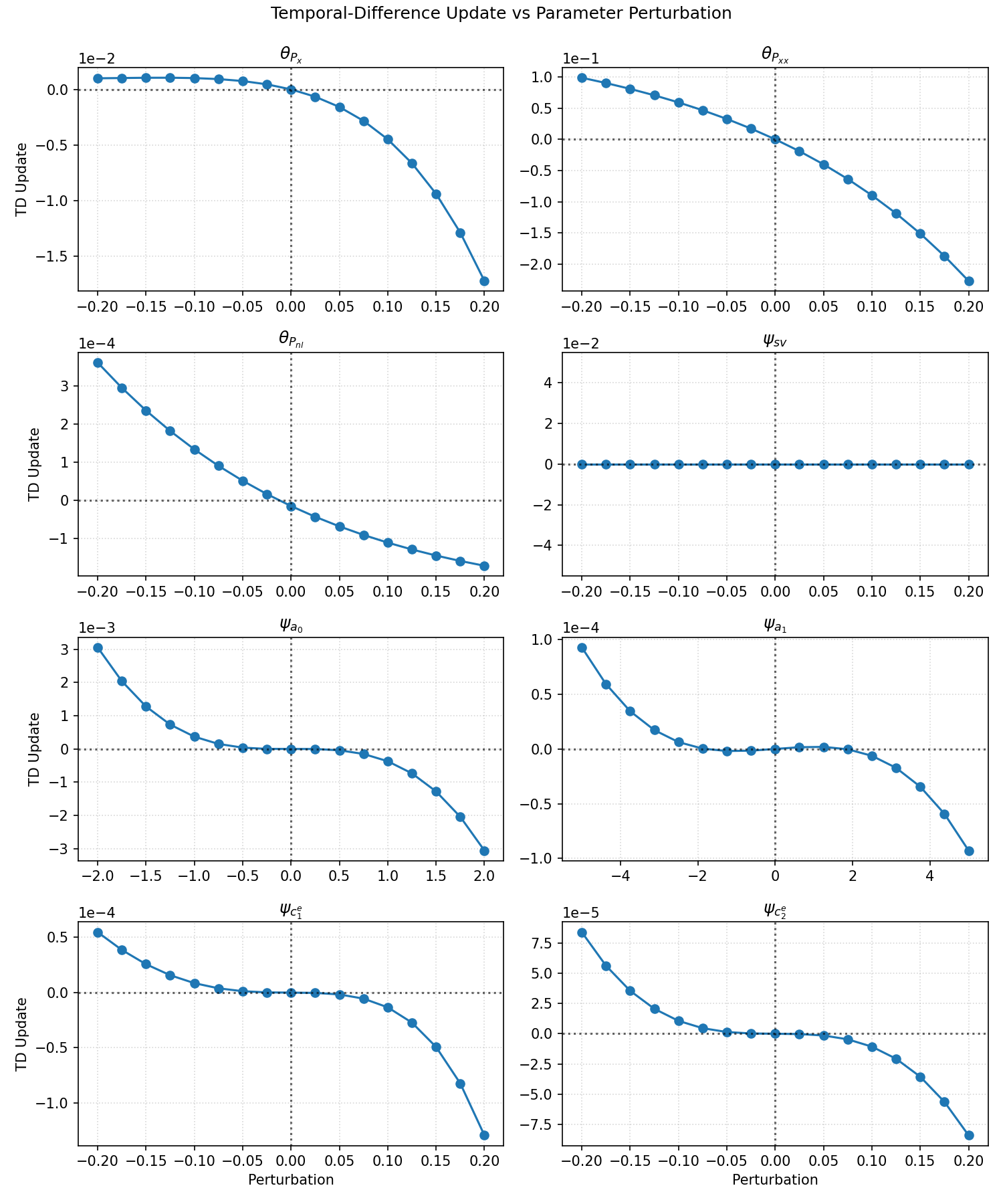}
    \caption{Temporal difference of parameters around optimum (zeros in the plots above).}
    \label{fig:mv_portfolio optimal grad}
\end{figure}

\clearpage
\section{Proofs}

\subsection{Proof of Proposition \ref{prop: markovian policy}}
\label{app: oce markov}

We adopt three steps to prove the proposition: (i) derivation of the augmented value function \eqref{eq: J}, (ii) Markovian optimality in the augmented SDE \eqref{eq: augmented sde}-\eqref{eq: J}, and (iii) equivalence between optimal policies in the original SDE \eqref{eq: sde}-\eqref{eq: J0} and the augmented SDE \eqref{eq: augmented sde}-\eqref{eq: J}.

\paragraph{Derivation of the augmented value function.} \label{app: derive aug value} Suppose the risk measure $U$ is expectation and the discount factor for rewards is zero as in the conventional finite-horizon RL setting. The entropy-regularized value function should be
\begin{align*}
    J^\pi(t, x, b_0, b_1) &= \E\bigg[ \int_t^T r_{\aug}(s,X_s^\pi,B_{0,s}^\pi,B_{1,s}^\pi,a_s^\pi) \rd s + h_{\aug}(X_T^\pi, B_{0,T}^\pi, B_{1,T}^\pi) \\
    &\qquad\quad - \tau \int_t^T B_{1,s}^\pi \log \pi(a_s^\pi\mid \gF_s^X) \rd s \mid X_t^\pi = x, B_{0,t}^\pi = b_0, B_{1,t}^\pi = b_1 \bigg] \\
    &= \E\big[ \varphi(B_{0,T}^\pi + B_{1,T}^\pi h(X_T^\pi)) \mid X_t^\pi = x, B_{0,t}^\pi = b_0, B_{1,t}^\pi = b_1 \big] + \tau b_1 \Ent(\pi;t,x) \\
    &= \E\bigg[ \varphi\bigg( \underbrace{b_0 + b_1 \int_t^T e^{-\delta(s-t)}r(s, X_s^\pi, a_s^\pi) \rd s}_{B_{0,T}^\pi} \\
    &\qquad\qquad + \underbrace{b_1 e^{-\delta(T-t)}}_{B_{1,T}^\pi} h(X_T^\pi) \bigg) \mid X_t^\pi = x, B_{0,t}^\pi = b_0, B_{1,t}^\pi = b_1 \bigg] + \tau b_1 \Ent(\pi;t,x) \\
    &= \E[\varphi(b_0 + b_1 Z^{\pi}(t, x))] + \tau b_1 \Ent(\pi;t,x),
\end{align*}
which concludes the derivation of \eqref{eq: J}. Note that the entropy regularization $\Ent(\pi; t, x)$ is still $\delta$-discounted across time to align with its original definition, despite absence of discount in rewards.

\paragraph{Markovian optimality in the augmented SDE.} The Markovian optimality follows naturally from the conventional RL literature \citep{puterman2014markov}. More detailed examples can be found in \cite{jia2022policy1,jia2022policy2,jia2023q}.

\paragraph{Equivalence between optimal policies.} \label{app: equiv opt val policy} By the definition of OCE, we have
\begin{align*}
    J_0^*(t,x) &= \max_{\pi \in \Pi_{\HD}} \left\{ \max_{b\in \R} \{ b + \E[\varphi(Z^{\pi}(t, x) - b)] \} + \tau \Ent(\pi;t,x) \right\} \\
    &= \max_{b\in \R} \left\{ b + \max_{\pi \in \Pi_{\HD}} \{ \E[\varphi(Z^{\pi}(t, x) - b)] + \tau \Ent(\pi;t,x) \} \right\} \\
    &= \max_{b\in \R} \left\{ b + \max_{\pi \in \Pi_{\HD}} J^{\pi}(t,x,-b,1) \right\} \\
    &= \max_{b\in \R} \left\{ b + \max_{\pi \in \Pi_{\mkv}} J^{\pi}(t,x,-b,1) \right\} \\
    &= \max_{b\in \R} \left\{ b + J^{*}(t,x,-b,1) \right\}.
\end{align*}
Denoting $\pi^*$ as the Markovian optimal policy that achieves $J^*$ in the last equality above, and letting $b^* = \argmax_{b\in \R} \left\{ b + J^{*}(t,x,-b,1) \right\}$, we can construct a history-dependent policy of the original SDE \eqref{eq: sde}-\eqref{eq: J0} as follows:
\begin{align}
\label{eq: pi_0^*}
    \pi^*_0(a\mid \gF_s^X) = \pi^*(a\mid s, X_s, B_{0,s}, B_{1,s}),\quad s\in [t,T],
\end{align}
where $(X_s, B_{0,s}, B_{1,s})$ follows the augmented SDE \eqref{eq: augmented sde} starting from $X_t = x, B_{0,t}=-b^*, B_{1,t}=1$.
Hence, we have
\begin{align*}
    J_0^{\pi^*_{0}}(t, x) &= \max_{b\in \R} \big\{ b + \E\big[\varphi\big(Z^{\pi^*_{0}}(t, x) - b\big)\big] \big\} + \tau \Ent(\pi;t,x) \\
    &\geq b^* + \E\big[\varphi\big(Z^{\pi^*_{0}}(t, x) - b^*\big)\big] + \tau \Ent(\pi;t,x) \\
    &= b^* + J^{\pi^*}(t, x, -b^*, 1) \\
    &= b^* + J^{*}(t, x, -b^*, 1) \\
    &= J^{*}_0(t,x),
\end{align*}
where we have used the definitions of $J_0^\pi$ as \eqref{eq: J0} and $J^\pi$ as \eqref{eq: J}, the relationship between $\pi_0^*$ and $\pi^*$ as \eqref{eq: pi_0^*}, and the property of $J_0^*$ derived above. This means $\pi^*_{0}$ is the optimal policy for the original SDE \eqref{eq: sde}-\eqref{eq: J0}, which concludes the proof.

\begin{remark}
    The optimal policy $\pi_0^*$ depends on the starting time $t$ of the SDE \eqref{eq: sde}. In other words, for $0 < t_1 < t_2 < s < T$, the optimal policy at time $s$ that attains $J_0^*(t_1,\cdot)$ is different from the one that attains $J_0^*(t_2,\cdot)$. This is because the augmented MDP includes the state $b_1$ which implicitly tracks the time duration from start. For notational simplicity, we omit the policy's dependence on the starting time $t$; we will always refer to an optimal policy along with an optimal value at a specific time $t$.
\end{remark}

\subsection{Proof of Theorem \ref{thm: J q martingale}}

If $\hat{J}^*=J^*$ and $\hat{q}^*=q^*$ are respectively the optimal value function and the optimal q-function, applying It\^o's lemma to the process $J^*(s, X_s^\pi, Y_s^\pi, e^{-\delta(s-t)})$ yields
\begin{align*}
    &J^*(s, X_s^\pi, Y_s^\pi, e^{-\delta (s-t)}) - J^*(t, x, 0, 1) - \int_t^s q^*(u, X_u^\pi, Y_u^\pi, e^{-\delta (u-t)}, a_u^\pi) \rd u \\
    &= \int_t^s \partial_x J^*(u, X_u^\pi, Y_u^\pi, e^{-\delta (u-t)}) \cdot \sigma(u, X_u^\pi, a_u^\pi) \rd W_u,
\end{align*}
which is an $(\gF, \PB)$-martingale.

Conversely, we consider the case where \eqref{eq: j q martingale} is an $(\gF, \PB)$-martingale, The second constraint in \eqref{eq: constraints} implies that $\hat{\pi}^*(a\mid t, x, b_0, b_1) := \exp\{ \hat{q}^*(t,x,b_0,b_1,a) / (\tau b_1) \}$ is a probability density function, and $\hat{q}^*(t, x, b_0, b_1, a) = \tau b_1 \log \hat{\pi}^*(a\mid t, x, b_0, b_1)$. Hence $\hat{q}^*(t, x, b_0, b_1, a)$ satisfies the second constraint in \eqref{eq: J q pi pi' constraint} with respect to the policy $\hat{\pi}^*$. It then follows from Theorem \ref{thm: J q pi pi' martingale} that $\hat{J}^*$ and $\hat{q}^*$ are respectively the value function and the q-function associated with $\hat{\pi}^*$. In addition, since $\hat{\pi}^*$ is the fixed point of the policy improvement iteration:
\begin{align*}
    \hat{\pi}^*(a\mid t, x, b_0,b_1) = \exp\left\{ \frac{\hat{q}^*(t,x,b_0,b_1,a)}{\tau b_1} \right\} = \frac{\exp\{ {\hat{q}^*(t,x,b_0,b_1,a)}/(\tau b_1) \}}{\int_{\gA} \exp\{ {\hat{q}^*(t,x,b_0,b_1,a)}/(\tau b_1) \} \rd a},
\end{align*}
we conclude that $\hat{\pi}^*$ is the optimal policy and thus $\hat{J}^*$ and $\hat{q}^*$ are the optimal value function and the optimal q-function, respectively.

\end{document}